\documentclass[letterpaper]{article} 
\usepackage{aaai2026}  
\usepackage{times}  
\usepackage{helvet}  
\usepackage{courier}  
\usepackage[hyphens]{url}  
\usepackage{graphicx} 
\usepackage{natbib}  
\usepackage{caption} 
\usepackage{algorithm}
\usepackage{algorithmic}

\usepackage{amsmath}
\usepackage{amsthm}
\usepackage{amsfonts}
\usepackage{pifont}
\usepackage{multirow}

\DeclareMathOperator*{\argmax}{argmax}

\newtheorem{theorem}{Theorem}
\newtheorem{lemma}{Lemma}

\usepackage{newfloat}
\usepackage{listings}
\DeclareCaptionStyle{ruled}{labelfont=normalfont,labelsep=colon,strut=off} 
\floatstyle{ruled}
\newfloat{listing}{tb}{lst}{}
\floatname{listing}{Listing}
\title{Learning Network Dismantling Without Handcrafted Inputs}
\author{
    Haozhe Tian\textsuperscript{\rm 1},
    Pietro Ferraro\textsuperscript{\rm 1},
    Robert Shorten\textsuperscript{\rm 1},
    Mahdi Jalili\textsuperscript{\rm 2},
    Homayoun Hamedmoghadam\textsuperscript{\rm 1}\thanks{Corresponding author: h.hamed@imperial.ac.uk.}
}
\affiliations{
    \textsuperscript{\rm 1}Dyson School of Design Engineering, Imperial College London, United Kingdom\\
    \textsuperscript{\rm 2}School of Engineering, RMIT University, Australia\\
    ht721@ic.ac.uk, p.ferraro@ic.ac.uk, r.shorten@ic.ac.uk, mahdi.jalili@rmit.edu.au, h.hamed@ic.ac.uk
}

\begin{document}

\hyphenpenalty 10000

\maketitle

\begin{abstract}
The application of message-passing Graph Neural Networks has been a breakthrough for important network science problems. However, the competitive performance often relies on using handcrafted structural features as inputs, which increases computational cost and introduces bias into the otherwise purely data-driven network representations. Here, we eliminate the need for handcrafted features by introducing an attention mechanism and utilizing message-iteration profiles, in addition to an effective algorithmic approach to generate a structurally diverse training set of small synthetic networks. Thereby, we build an expressive message-passing framework and use it to efficiently solve the NP-hard problem of Network Dismantling, virtually equivalent to vital node identification, with significant real-world applications. Trained solely on diversified synthetic networks, our proposed model---MIND: Message Iteration Network Dismantler---generalizes to large, unseen real networks with millions of nodes, outperforming state-of-the-art network dismantling methods. Increased efficiency and generalizability of the proposed model can be leveraged beyond dismantling in a range of complex network problems.
\end{abstract}

\begin{links}
    \link{Code}{https://github.com/HaozheTian/MIND-ND}
    \link{Extended version}{https://arxiv.org/pdf/2508.00706}
\end{links}

\section{Introduction}
Network dismantling is the problem of finding the sequence of node removals that most rapidly fragments a network into isolated components~\citep{braunstein2016network, ren2019generalized}. Finding dismantling solutions is equivalent to the identification of vital components of the network system, and has profound real-world applications, such as breaking criminal organizations by arresting the key members~\citep{ribeiro2018dynamical}, stopping epidemics with targeted vaccinations~\citep{kitsak2010identification, cohen2003efficient}, ensuring the resilience of healthcare systems via the key providers  \citep{lo2019quantification}, and preventing wildfires by securing critical locations~\citep{demange2025instantiating}. Figure~\ref{fig:sche} visualizes network dismantling of a real-world social network~\citep{guo2016novel} in action, where strategically removing a mere $7\%$ of nodes effectively breaks it into small components.

\begin{figure}[t]
    \centering
    \includegraphics[width=\linewidth]{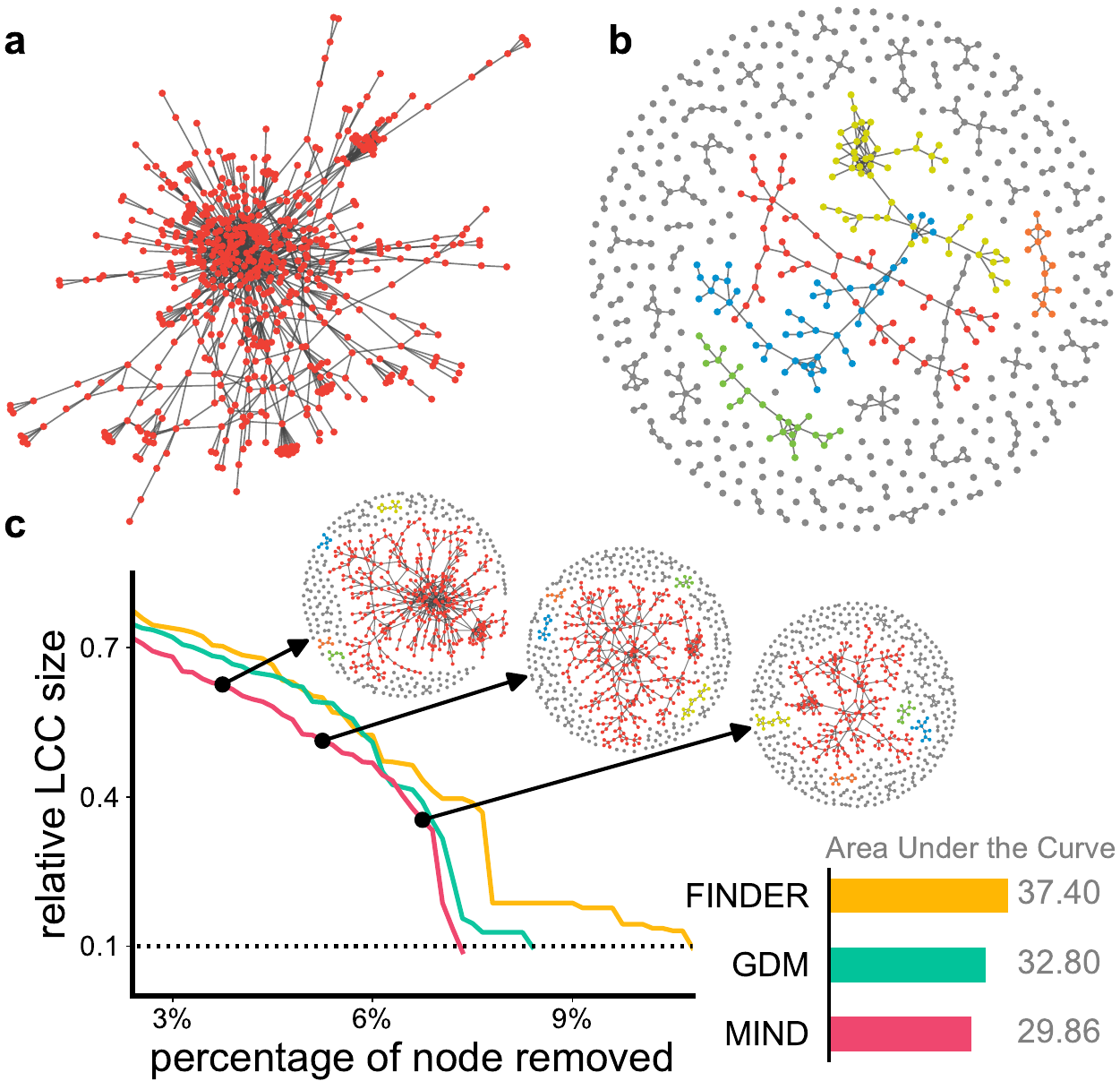}
    \caption{
    (\textbf{a}) The original social network from the FilmTrust project~\citep{guo2016novel} with 610 nodes. 
    (\textbf{b}) The dismantled network by MIND, down to a 10\% relative Largest Connected Component (LCC) size. 
    (\textbf{c}) Relative LCC size versus the fraction of nodes removed, comparing MIND with two state-of-the-art methods. (The 5 largest components are color-coded in network plots.)}
    \label{fig:sche}
\end{figure}

Despite the practical significance, only approximate solutions can be sought for network dismantling, due to the NP-hard nature of the problem~\citep{braunstein2016network}. Yet, the possibility of reaching a universal perception of structural roles, and the challenge of planning along the extreme breadth and depth of the search, have motivated the decades-long quest for better dismantling solutions.
The early solutions use node centrality metrics as heuristics~\citep{freeman1977set,wandelt2018comparative}, with advancements later made by theoretical solutions to more tractable proxy problems, including optimal percolation~\citep{morone2015influence}, graph decycling~\citep{braunstein2016network}, and minimum cut~\citep{ren2019generalized}.
Recent methods use Graph Neural Networks (GNNs) to learn vector representations of nodes through iterative message-passing, which scales linearly with the number of nodes and edges and is parallelizable on GPUs~\citep{velickovic2018graph, hamilton2017inductive}. The well-performing existing methods~\citep{fan2020finding, grassia2021machine} rely on handcrafted inputs to aid the inference of nodes' importance. 

We argue that using handcrafted input features i)~imposes significant computational overhead, especially for large-scale networks, and ii)~biases the learned dismantling strategy toward the predefined features whose effectiveness varies across network families.
To address these limitations, we propose Message Iteration Network Dismantler (MIND)---a model solely based on data-driven geometric learning without manually engineered features. 
While it is well-established that initializing GNNs with structural heuristics improves the performance~\citep{cui2022positional}, MIND achieves competitive performance through pure message-passing. 
Specifically, MIND i)~employs an expressive attention mechanism to replace the non-injective softmax normalization in existing Graph Attention Networks (GATs), and ii)~leverages node embeddings from all message-passing iterations to capture crucial structural information. We demonstrate that the design of MIND enables the estimation of complex structural roles, such as those given by combinatorial centrality and spectral embedding, which are proven essential for network dismantling ~\citep{wandelt2018comparative}, and empirically show that MIND can discover new, data-driven node features that outperform known dismantling heuristics.

MIND learns to identify critical nodes and substructures with Reinforcement Learning (RL) that is trained by dismantling small synthetic networks.
We introduce a training pipeline utilizing degree-preserving edge rewiring to systematically synthesize structurally diverse networks with varying levels of assortativity and modularity. Interactions with these diversified networks significantly enhance the policy’s ability to generalize to complex real-world networks.
Our proposed trained policy scales well to networks with well over 1 million nodes, achieving state-of-the-art dismantling performance using only the raw incidence information, without any handcrafted input.
The key contribution is the introduction of a pure geometric learning framework that can decipher the complex structural roles of network entities with surprising generalizability, resulting in the best performance to date on one of the most challenging network problems.

\section{Related Works}

\subsubsection{Reinforcement Learning}
RL solves combinatorial optimization problems by learning a policy that maximizes the expected cumulative reward of action sequences in the given state space~\citep{sutton1998reinforcement, konda1999actor}.  
This approach enables the optimization of non-differentiable objectives via sequential decision-making; e.g., reducing network connectivity by iterative node removals, which is infeasible through exhaustive search, and difficult for heuristic methods that do not generalize to the infinite possible network configurations. This necessitates learning from experience, where RL, especially with recent advances in Deep RL~\citep{mnih2013playing, khalil2017learning, haarnoja2018soft}, is a particularly effective solution.

\subsubsection{Graph Neural Networks}
GNNs learn representations of network structures through node embeddings that are iteratively refined by aggregating the messages in each node's neighborhood ~\citep{kipf2016semi, hamilton2017inductive, brody2022how}. A shared, learnable function transforms these messages, enabling generalization to networks of arbitrary sizes. Theoretically, certain GNN architectures can distinguish almost all non-isomorphic networks even without initial node features, provided the learned function is a universal approximator~\citep{kipf2016semi, dai2016discriminative, xu2018how, morris2019weisfeiler}. 
However, in practice, initializing node embeddings with constant or random values often degrades the performance compared to using handcrafted features~\citep{cui2022positional}. The latter facilitates convergence~\citep{oono2020graph}, but introduces bias to the learned embeddings (see the discussion on Fig.~\ref{fig:corr}), as nodes with similar initial features are placed close to each other in the embedding space.

\subsubsection{Network Dismantling via Machine Learning}%
GNN-based embedding has significantly advanced network dismantling. \citet{fan2020finding} use RL to train GNNs from experience in dismantling small random networks, but incorporate handcrafted global structural features into the GNN embeddings. \citet{grassia2021machine} train their model on brute-force optimal dismantling sequences found for small networks, yet rely on a set of input node features (degree, neighborhood degree statistics, $k$-coreness, and clustering coefficient) to be calculated before being applied to dismantle a network. \citet{khalil2017learning} show that without manually engineered node features, GNNs can solve other network combinatorial optimization problems, e.g., minimum vertex cover, max-cut, and the traveling-salesperson problem, yet, to the best of our knowledge, no existing method has achieved competitive dismantling performance in this setting.

\section{Network Dismantling as an RL Problem}
Let $\mathcal{G}$ denote the universe of all possible networks and $P_\mathcal{G}$ be the distribution from which a network (or graph) is drawn: $G_0=(V_0, E_0)\sim P_\mathcal{G}$, where $V_0$ is the set of nodes and $E_0$ is the set of edges between the nodes. At each step $t=0,\cdots,|V_0|-1$, the network dismantling policy $\pi(v_i|G_t)$ observes $G_t=(V_t, E_t)$ and outputs a distribution over $v_i\in V_t$, from which a node is drawn $v_t\sim\pi(v_i|G_t)$ and removed from $G_t$ (along with its incident edges), which we formulate as $G_{t+1}=G_t\setminus \{v_t\}$. With slight abuse of notation, we simplify $v_t\sim\pi(v_i|G_t)$ as $v_t=\pi(G_t)$. The standard objective for the network dismantling problem is to minimize the area under the curve (AUC) of the relative Largest Connected Component (LCC) size of the network over the sequence of node removals, which we use to formulate policy optimization:
\begin{align}
\min_\pi \mathbb{E}\left[\sum_{t=0}^{|V_0|-1}\frac{\text{LCC}(G_0\setminus\{v_0, \cdots, v_t\})}{|V_0|}\right],
\label{eq_auc}
\end{align}
where $\text{LCC}(.)$ returns the relative size of the LCC. The optimization problem in \eqref{eq_auc} can be rewritten as the sum of rewards:
\begin{align}
\max_\pi \mathbb{E}\left[\sum_{t=0}^{|V_0|-1} r_t\right], r_t = -\frac{\text{LCC}(G_t\setminus\{v_t\})}{|V_0|}.
\label{eq_objrl}
\end{align}
Since $G_{t+1}$ depends only on the current $G_t$ and $v_t$, the problem in \eqref{eq_objrl} forms a Markov Decision Process (MDP) that can be solved using RL in a data-driven manner. We also follow the standard definition of the state-action value function:
\begin{align}
    Q(G_t, v_i) = r_t+\mathbb{E}\left[\sum_{k=t+1}^{|V_0|-1} r_k\right],
\end{align}
where $Q(G_t, v_i)$ denotes the expected cumulative return (i.e., expected future AUC) starting with the removal of node $v_i$ in network $G_t$ and thereafter following the policy $\pi$.

Specifically, we solve \eqref{eq_objrl} using an Actor-Critic RL algorithm, where the actor corresponds to the dismantling policy $\pi(v_i|G_t)$ and the critic to the state-action value function $Q(G_t, v_i)$. Each training iteration consists of two sub-processes: \textit{value estimation} and \textit{policy improvement}. The Actor-Critic framework features \textit{experience replay}, which increases sample efficiency; each training iteration is performed on a randomly sampled batch of historical state transitions $\mathcal{B} \subseteq \{(G_0, v_0, r_0, G_{1}),\dots, (G_t, v_t, r_t, G_{t+1})\}$. In \textit{value estimation}, MIND estimates $Q(G_t, v_i)$ with the Bellman equation:
\begin{align}
Q(G_t, v_t) \approx \mathbb{E}_{\mathcal{B}}\left[r_t +   Q(G_{t+1}, \pi(G_{t+1})) \right].
\end{align}
Then, the \textit{policy improvement} updates policy $\pi$ by solving:
\begin{align}
\hat{\pi}=\argmax_\pi \mathbb{E}_{G_t \in \mathcal{B}} \left[ Q(G_t, \pi(G_t)) \right].
\label{eq_piupdate}
\end{align}
The batch $\mathcal{B}$ is sampled from trajectories generated on networks from the same distribution $P_\mathcal{G}$ and by the same policy $\pi(v_i|G_t)$ (as in~\eqref{eq_objrl}), therefore, the maximization in Eq.~\eqref{eq_piupdate} corresponds to a Monte Carlo approximation of the original objective in~\eqref{eq_objrl}.

\section{Methodology}
To learn the representation of complex networks that is generalizable across all networks in $\mathcal{G}$, MIND employs a GNN-based RL framework, where both the state-action value function $Q(G_t, v_i)$ and the policy $\pi(v_i|G_t)$ are parameterized by encoder-decoder neural networks. The encoder GNNs take the adjacency representation of $G_t$ as input and extract node embeddings $z_i$, which capture the structural role of each node $v_i \in V_t$. The decoders then map each $z_i$ to a scalar score, i.e., the state-action value of removing $v_i$ in the $Q(G_t, v_i)$ decoder, and the probability of selecting $v_i$ for removal in the $\pi(v_i | G_t)$ decoder. Since the encoders use a permutation-invariant GNN and the decoders are shared across all nodes, this architecture naturally handles networks of varying sizes and calculations of $Q(G_t, v_i)$ and $\pi(v_i|G_t)$ for any $(G_t, v_i)$ pair.

\subsection{GNN Encoder}
To learn network representations not biased by the selection of handcrafted node features, the GNN encoder of MIND initializes each node $v_i \in V_t$ with a set of $H$ all-ones vectors, $\{e_i^h = \mathbf{1}_F | h = 1, 2, \dots, H\}$, each $e_i^h$ serving as a \emph{head}, allowing for simultaneous encoding of diverse structural information. We propose a GNN encoder that incorporates two mechanisms (detailed in this section) that enable effective network representation learning with simple all-ones initialization. 

\subsubsection{All-to-One Attention Mechanism}
At each message-passing iteration, the embedding vector $e_i^h$ of node $v_i$ is updated using the following rule:
\begin{align}
    \label{eq_mp}
    \hat{e}_i^h = \alpha_i^h W_\sigma^h e_i^h + \sum_{j \in \mathcal{N}(i)} \alpha_{i, j}^h W_\nu^h e_j^h,
\end{align}
where $W_\sigma^h, W_\nu^h \in \mathbb{R}^{F \times F}$ are learnable weight matrices. We propose the attention mechanism (MIND-AM) below to calculate the coefficients $\alpha_i^h$ and $\alpha_{i, j}^h$:
\begin{align}
    &\alpha_i^h = \mathrm{MLP}_\sigma^h\left(\left[\mathop{\|}_{h=1}^H W_\sigma^h e_i^h\right]\right), \tag{MIND-AM}\label{eq_am}\\
    &\alpha_{i,j}^h = \mathrm{MLP}_\nu^h\left(\left[\mathop{\|}_{h=1}^H W_\sigma^h e_i^h\right]+\left[\mathop{\|}_{h=1}^H W_\nu^h e_j^h\right]\right), \notag
\end{align}
where $\left[\mathop{\|}_{h=1}^H x^h\right]$ is a vector concatenation as $[x^1 \| \cdots \| x^H]$, and $\mathrm{MLP}_\sigma^h, \mathrm{MLP}_\nu^h: \mathbb{R}^{HF} \to (0, 1)$ are head-specific neural networks with $\mathrm{sigmoid}$-squashed outputs. Equation~\eqref{eq_mp} uses attention coefficients $\alpha^h$ to selectively aggregate messages from different neighbors, similar to the state-of-the-art GATs~\citep{velickovic2018graph, brody2022how}. However, GATs do not learn when $e_i^h = \mathbf{1}_F$ for all $i$ and $h$, since softmax-normalization of $\alpha^h$ keeps node embeddings identical over message-passing iterations (as demonstrated in Fig.~\ref{fig:fig_abla}).

The idea behind MIND-AM is to employ an attention mechanism that i) eliminates the need for softmax normalization of $\alpha^h$ and thus preserves injectivity over the multiset $\{e_j^h : j \in \mathcal{N}(i)\}$, and ii) controls the explosion of $|e_i^h|$ without explicit normalization. Equations (MIND-AM) achieve the above by computing each head's attention coefficient $\alpha^h$ using features from all heads. Thereby, our encoder automatically learns to leverage node information (e.g., local degree-like features) captured in other heads to normalize messages and prevent feature explosion.

\subsubsection{Message Iteration Profiles} Let $e_i^{(k)}$ denote the embedding vector of node $v_i$ after the $k$-th message-passing iterations, calculated by concatenating the embeddings across all heads: $e_i^{(k)}=\left[\mathop{\|}_{h=1}^H e_i^h\right]$, at layer $k$.  MIND computes the Message Profile \eqref{eq_fe} as the final node embedding $z_i$, i.e., the profile of embeddings over all message-passing iterations:
\begin{align}
     z_i = \mathrm{MLP}_\zeta\left(\left[\mathop{\|}_{k=1}^K e_i^{(k)}\right]\right),
\tag{MIND-MP}\label{eq_fe}
\end{align}
where $\mathrm{MLP}_\zeta^h$ is a shared neural network between all nodes.

The first motivation for \ref{eq_fe} is the well-known issue of over-smoothing in node embeddings caused by iterative message-passing~\citep{li2018deeper, oono2020graph}. In Appendix~A, \textbf{Theorem~1}, we show that the embeddings $e_i^{(k)}$ for all nodes $v_i\in V_t$ tend to converge to the primary eigenvector of the message-passing operator as $k$ increases. \ref{eq_fe} retains local structural information from early iterations, thereby preserving the diversity of node embeddings.

The second motivation for \ref{eq_fe} is to extract crucial structural information that can only be obtained by jointly considering all message-passing iterations. Although $e_i^{(k)}$ converges as $k$ increases, nodes converge at different rates, depending on their centrality~\citep{hage1995eccentricity}, as more central nodes begin aggregating information from the entire network earlier. Further theoretical insights into MIND’s expressiveness are provided in Appendix~B, where we show that by learning the message-passing operator in \textbf{Lemma~1}, MIND can approximate the Fiedler vector, a widely-used spectral heuristic in network dismantling literature~\citep{wandelt2018comparative, grassia2021machine}.

\subsection{NN Decoder}
In addition to the node embeddings $z_i$, which encode the structural roles of individual nodes, we introduce a synthetic \textit{omni-node} $v_o$ to $G_t$. Each node $v_i \in V_t$ is connected to $v_o$ via a directed edge, enabling one-way message-passing from all nodes to $v_o$. This design allows the resulting embedding $z_o$ from the GNN Encoder to aggregate information from the entire network and represent the global state of $G_t$. Both the omni-node and individual node embeddings are passed to the decoders to enable state-aware decision-making. In particular, as formulated below, the $Q$ decoder learns to estimate the remaining dismantling AUC, while the $\pi$ decoder  predicts the relative importance of each node for the next removal step:
\begin{align}
\begin{split}
Q(G_t, v_i) &= \mathrm{MLP}_\theta\left([z_i \| z_o]\right), \\
\pi(v_i|G_t) &= \mathrm{MLP}_\phi\left([z_i \| z_o]\right).
\end{split}
\end{align}
By leveraging both local information (through $z_i$) and global information (through $z_o$), the learned network dismantling policy can perform long-term planning and adapt based on the current state of dismantling. The neural networks $\mathrm{MLP}_\theta$ and $\mathrm{MLP}_\phi$ are shared across all nodes, enabling MIND to generalize across networks of varying sizes.

\subsection{Systematically Diversified Training Networks}
Our goal is to train a universal dismantler that generalizes across all $G_0 \in \mathcal{G}$. So, it is essential to train on diverse network configurations. For this purpose, the common practice is to generate networks of different sizes (and densities) using random graph models. The significance of the famous graph models, to an extent however, does not reflect their representativeness of the real (or possible) networks (and arguably has more to do with tractable mathematical properties). Here, we propose a systematic procedure to generate random training networks that better reflect the structural diversity of real-world networks. In short, the proposed procedure takes small (100-200 nodes) random networks of different degree distributions and introduces different levels of modularity and degree-assortativity by randomizing the configurations (keeping the degree sequences fixed); this also attenuates the geometrical properties inherited from the graph generation models. 

We first synthesize 10,000 random networks using Linear Preferential Attachment (LPA) \citep{newman2018networks}, Copying Model \citep{kumar2000stochastic}, and Erdos-Renyi (ER) \citep{erdds1959random} models. To enhance the structural diversity, we apply degree-preserving edge rewirings to induce different types of node mixings. Specifically, we perform random edge rewirings that either favor or discourage connections between nodes with similar labels, either by degree to create varying levels of degree assortativity (assortative, uncorrelated, disassortative), or randomly to induce varying levels of modularity (modular, random, and multipartite). See Appendix~C for further details on the generation process.

\subsection{Entropy-Regularized Policy Learning} To train MIND for solving \eqref{eq_objrl}, we perform multiple dismantling episodes, each beginning with a network randomly sampled from the training set. The specific RL algorithm (detailed in Appendix~D, Algorithm~1) is based on Soft Actor-Critic (SAC)~\citep{haarnoja2018soft}, chosen for its high sample efficiency and its ability to encourage effective exploration via entropy regularization. However, unlike the original SAC, which handles continuous action spaces via Monte Carlo sampling, the action space $V_t$ here is discrete, allowing MIND to directly compute the expectation of the $Q$-value under the current policy for each $G_t$ as:
\begin{align}
\begin{split}
\mathbb{E}_{\pi}\left[Q(G_t, v_t)\right] = \sum_{v_i \in V_t} \pi(v_i|G_t) Q(G_t, v_i).
\end{split}
\end{align}

\section{Experiments}

\begin{figure*}[t]
    \centering
    \includegraphics[width=\linewidth]{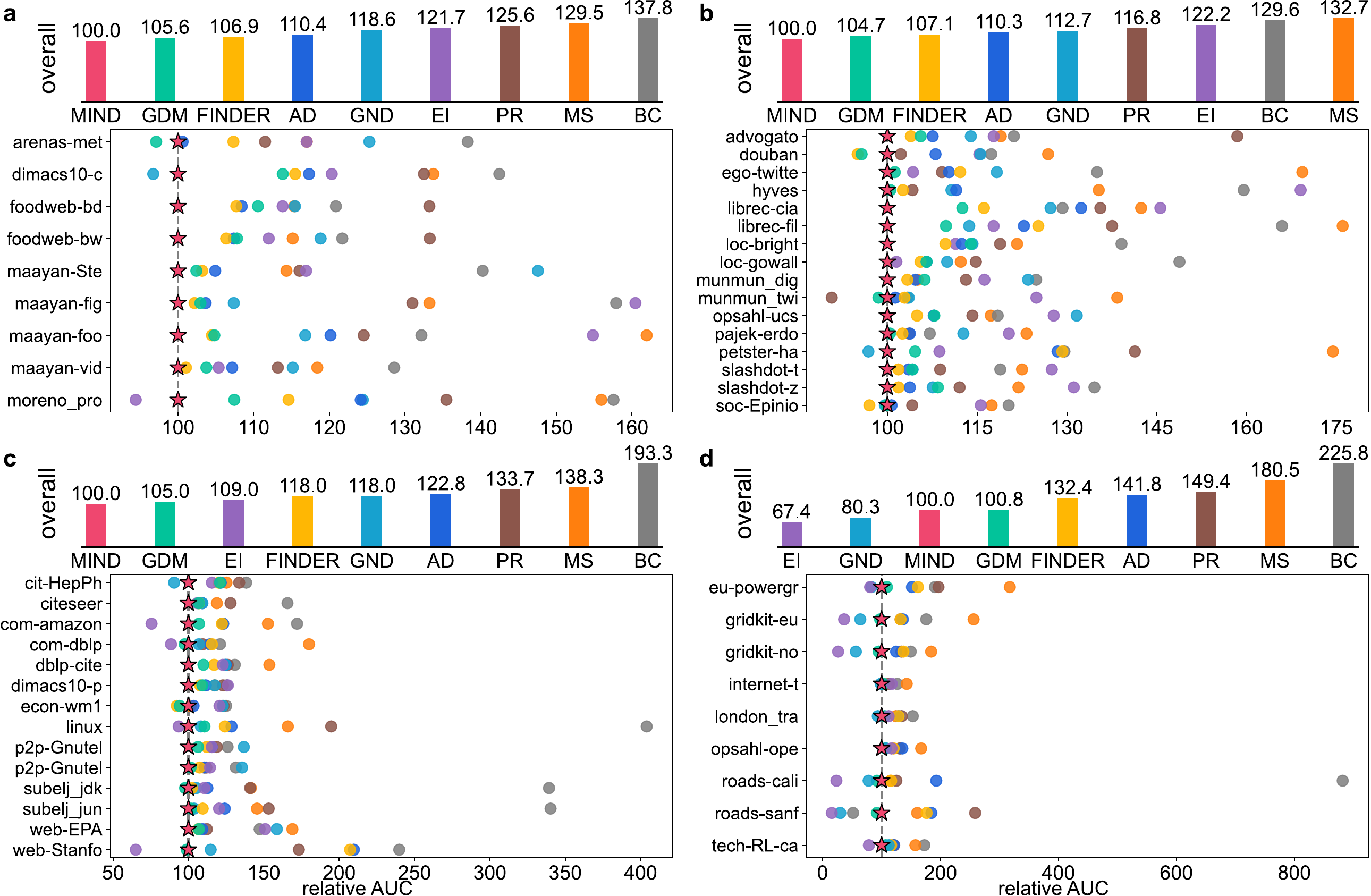}
    \caption{Dismantling performance of MIND and the baseline methods on (\textbf{a}) biological, (\textbf{b}) social, (\textbf{c}) information, and (\textbf{d}) technological networks. The scatter plots display the AUC of dismantling for all methods normalized relative to that of MIND at 100 (AUC above 100 denotes worse performance than MIND). The bar plots summarize the overall performance of the methods in each network domain, with shorter bars corresponding to lower average AUC and thus stronger dismantling performance.}
    \label{fig:main_res}
\end{figure*}

We compare the performance of MIND with a comprehensive set of baseline methods on both real-world and synthetic networks~\citep{braunstein2016network, clusella2016immunization, ren2019generalized, fan2020finding, grassia2021machine}. The baselines represent both the classic methods and the state-of-the-art, identified in the recent review by \citet{artime2024robustness}, and categorized as i) \textit{Centrality Heuristics}: Adaptive Degree (AD), Betweenness Centrality (BC), and PageRank (PR); ii) \textit{Approximate theory methods}: Min-Sum (MS), Explosive Immunization (EI), and Generalized Network Dismantling (GND); and iii) \textit{Machine Learning methods}: FINDER and GDM. MIND has the lowest computational complexity among machine learning-based dismantling methods (see Table~\ref{tb_complexity}), as a result of not requiring the computation of handcrafted node features for embedding initialization. Compared to the approximated theory baselines, MIND is also the most computationally efficient, except for the EI method only on dense networks, where $|E|$ asymptotically grows faster than $|V|\log|V|$.
All baselines are implemented following their respective references (readers may refer to the summary in Table~1 of \citep{artime2024robustness}). MIND is trained over 8 million dismantling episodes, each initialized with a random selection from the training set of 10,000 small synthetic networks. The detailed training setup of MIND is provided in Appendix~D.

\subsection{Result on Real Networks}
We evaluate MIND on real-world networks across four domains, namely, biological, social, information, and technological---covering a wide range of properties and sizes from 128 to 1.4 million nodes (summarized in Appendix~E, Table~4).
Figure~\ref{fig:main_res} reports the AUC of the dismantling curve for all methods, normalized relative to MIND for each network; the bar plots summarize the overall performances in each domain. The detailed relative AUC values are provided in Appendix~F.

\begin{table}[t]
\centering
\begin{tabular}{lc}
\hline
\multicolumn{1}{c}{Method} & \multicolumn{1}{c}{Complexity} \\ \hline
MS     & $\mathcal{O}(|V|\log |V|)+\mathcal{O}(|E|)$ \\
EI     & $\mathcal{O}(|V|\log |V|)$  \\
GND    & $\mathcal{O}(|V|\log^{2+\epsilon}|V|)$ \\
FINDER & $\mathcal{O}(|V|\log |V|+|E|)$  \\
GDM    & $\mathcal{O}(|V|\langle d^2\rangle + |E|)$\\ 
MIND   & $\mathcal{O}(|V|+|E|)$  \\ \hline
\end{tabular}
\caption{Computational complexity of methods assuming adjacency list representation of $G = (V, E)$. ($\langle d^2\rangle$ is the second moment of degree.)}
\label{tb_complexity}
\end{table}

The top three methods, ranked by overall performance across all networks, are MIND (100.0), GDM (104.13), and EI (107.96). The results demonstrate that although other machine learning baselines take advantage of handcrafted inputs, MIND consistently achieves stronger performance across all domains. This highlights that handcrafted initial embeddings, despite boosting the GNN training, do not inherently yield strong dismantling performance. In contrast, MIND, empowered by our proposed MIND-AM and MIND-MP mechanisms (see the GNN Encoder section), is able to identify structurally vital nodes purely from adjacency representation, resulting in an effective network dismantling policy. In technological networks (Fig.~\ref{fig:main_res}d), MIND slightly underperforms EI and GND, but still outperforms all other methods, including machine learning baselines (GDM and FINDER). This can be attributed to the limited reach of GNN message-passing in technological networks with very large diameters (e.g., over $100$ for gridkit-eupowergrid and gridkit-north\_america).

\subsection{Result on Synthetic Networks}
\begin{figure}[t]
    \centering\includegraphics[width=\linewidth]{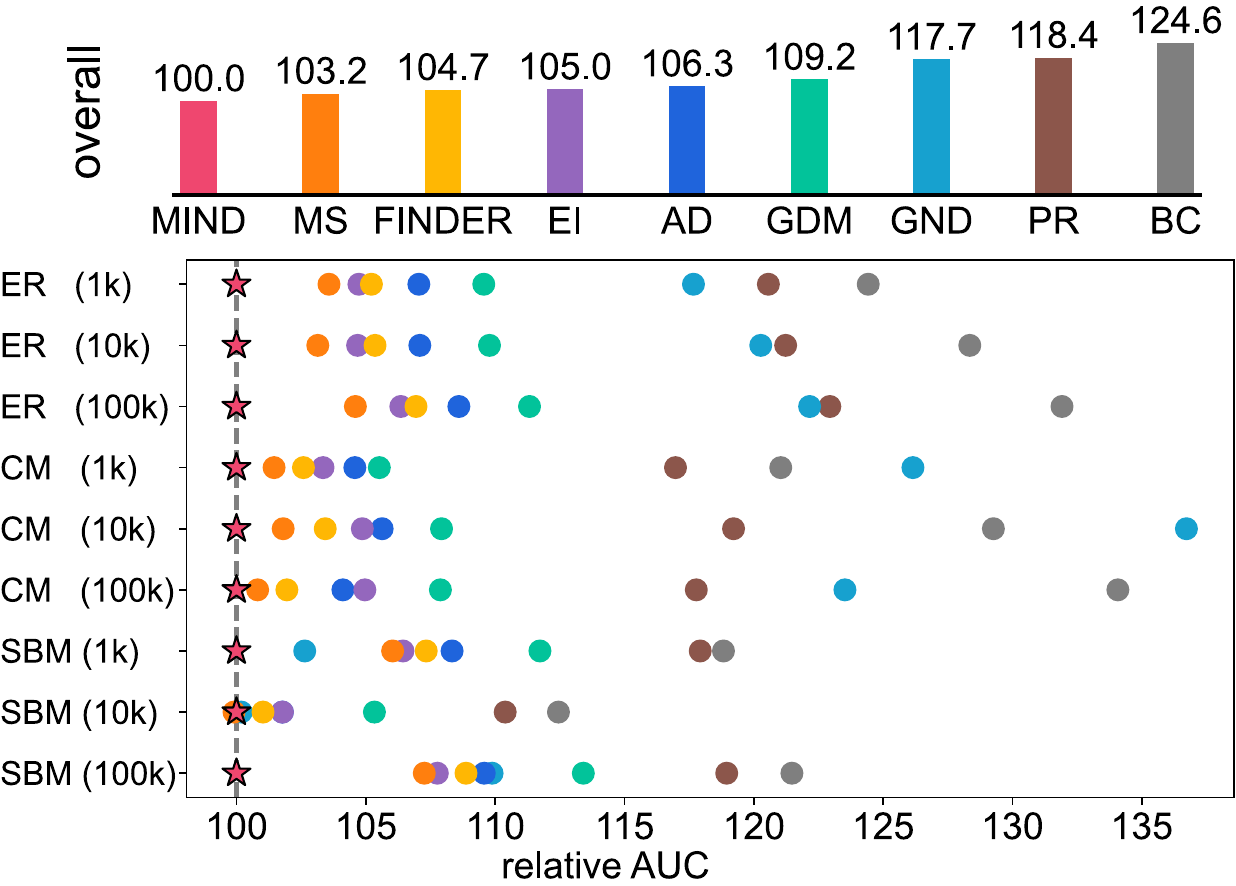}
    \caption{Dismantling performance of MIND and the baseline methods on synthetic networks (ER, CM, and SBM) of varying sizes. The scatter plot compares the dismantling performance of all methods normalized for each network relative to MIND, and the bar plot summarizes the overall performance. The AUCs are averaged over 10 realizations.}
    \label{fig:fig_synth}
\end{figure}

\begin{figure}[t]
    \centering
    \includegraphics[width=\linewidth]{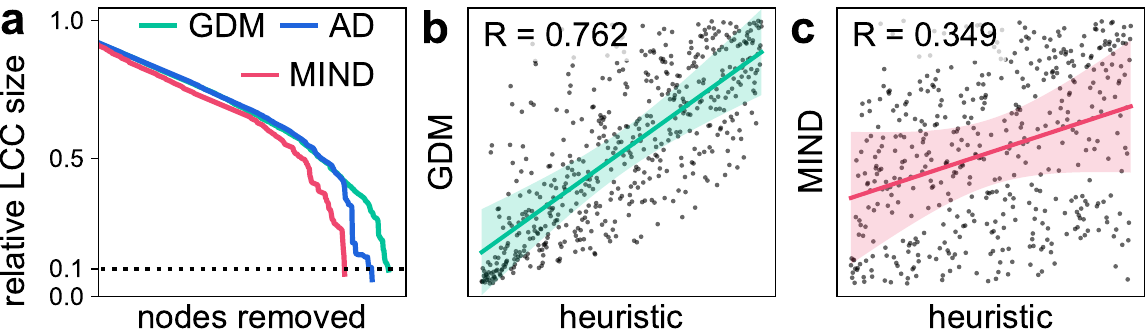}
    \caption{(\textbf{a}) Relative LCC size during the dismantling of an ER network with $1$~k nodes. We compare the node dismantling sequence derived (by PCA) from a set of heuristics with those generated by (\textbf{b}) GDM and (\textbf{c}) MIND; Spearman rank correlation coefficient $R$ and the regression (solid line) with confidence interval (shaded area) are shown on the plots. The heuristic removal sequence is derived from the principal component of GDM’s input node features.}
    \label{fig:corr}
\end{figure}

We evaluate MIND on synthetic networks generated by the widely-adopted protocols in prior studies: (i) ER with average degree $\langle d\rangle=4$, (ii) Configuration Model with $\langle d\rangle=4$ and degree distribution $P(d)\sim d^{-2.5}$, and (iii) Stochastic Block Model with group size $100$, $p_{\text{intra}}=0.1$, and $p_{\text{inter}}=5/|V|$. From each model, we generate networks of sizes $1$~k, $10$~k, and $100$~k, and evaluate the average AUC over $10$ realizations. Figure~\ref{fig:fig_synth} shows scatter plots of the AUC of dismantling for all methods, normalized relative to MIND for each network type, and bar plots comparing the overall performances. Note the testing networks in Fig.~\ref{fig:fig_synth} differ from those that MIND is trained on, in both sizes and methods of generation.

From the results, we observe that MIND significantly outperforms the baselines, except for Stochastic Block Model with 1~$k$ nodes, where the inherent community structures enable the decycling-based method (MS) to achieve a comparable performance to MIND. Notably, although GDM uses the node degree as an input feature, the learned message-passing functions extract structural information that ultimately leads to worse performance than the simple AD across all synthetic networks. Since GDM is trained on the same types of synthetic networks, this suggests that it overfits to the specific training set and loses generalizability to similar structures. This highlights the better generalizability of the RL-based (FINDER and MIND) dismantling policies compared to the supervised learning approach (GDM).

For an ER network where nodes are structurally similar, the simple AD method performs considerably better than GDM (Fig.\ref{fig:corr}a). Following this observation, we investigate whether GDM's degraded performance may be due to an inherent bias towards its handcrafted input features: degree, neighborhood degree statistics, $k$-coreness, and clustering coefficient. Let $X\in \mathbb{R}^{N\times 4}$ be the corresponding feature matrix for the ER network. We combine the input features by projecting $X$ onto the principal eigenvector of $\frac{1}{N}X^\top X$, and order the nodes accordingly to obtain a heuristic dismantling sequence. 
Figure~\ref{fig:corr}b shows a significant correlation between the dismantling sequence of GDM and that of its input features. In contrast, MIND dismantling has weak to no correlation with the heuristic dismantling sequence (Fig.~\ref{fig:corr}c). This corroborates that MIND gains performance by learning the underlying structural importance beyond the standard node heuristics.

\subsection{Ablation Studies}

\begin{figure}[t]
    \centering
    \includegraphics[width=\linewidth]{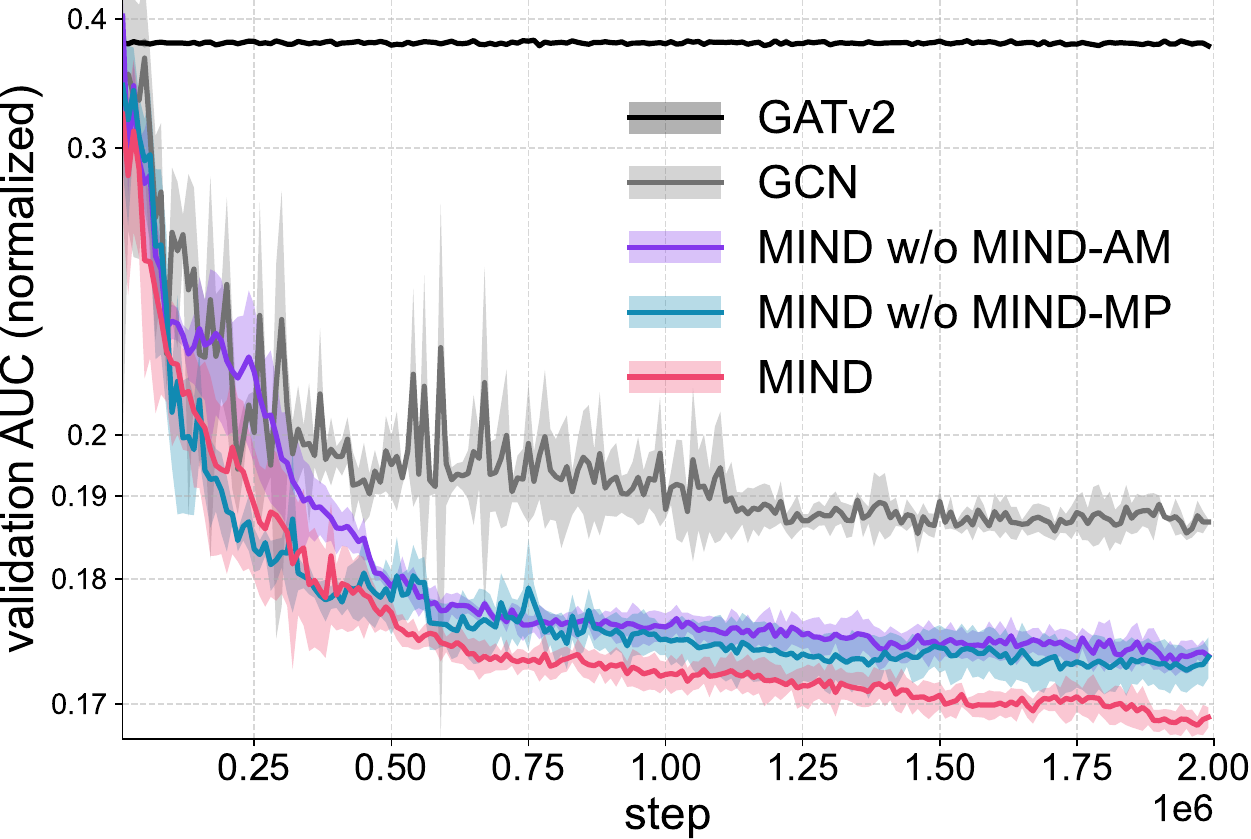}
    \caption{The validation AUC during training (mean$\pm$std). MIND is compared against: i) GATv2; ii) GCN; iii) MIND without MIND-AM (the all-to-one attention mechanism); and iv) MIND without MIND-MP (the message-profile over iterations).}
    \label{fig:fig_abla}
\end{figure}

\subsubsection{GNN Design}
The effectiveness of the proposed MIND-AM and MIND-MP is empirically verified via ablation experiments, where we remove each design component and observe the changes in validation performance (the AUC of dismantling) during training. 
To calculate the validation AUC, we conduct small-scale tests and take the average performance over 20 synthetic networks generated by LPA, ER, and Watts-Strogatz \citep{watts1998collective} models every 10,000 training steps.
We also include the original GATv2 and GCN in the comparison, using them as the message-passing operators while keeping the rest of the MIND framework unchanged (e.g., employing message-iteration profiles instead of the final embedding). The results are shown in Fig.~\ref{fig:fig_abla}, as the mean (solid line) and standard deviation (shaded area) of validation AUC during training over 5 independent runs.

The results demonstrate that MIND, even after removing the architectural designs proposed in this paper, still outperforms the existing GNN baselines. The original GATv2 fails to learn when the initial node embeddings are constants, due to its reliance on softmax-normalized attention coefficients. While GCN is able to learn, it achieves suboptimal performance and exhibits large fluctuations. Removing MIND-AM or MIND-MP degrades the performance of MIND. We perform t-test on the converged AUC values shown in Fig.~\ref{fig:fig_abla} for MIND against its two ablated variants, obtaining $p = 8.1 \times 10^{-4}$ and $p = 4.5\times 10^{-2}$, respectively, highlighting the effectiveness of the all-to-one attention mechanism to allow learnable normalization of the messages, and the benefit of utilizing information from all message-passing iterations to extract deeper structural insights.

\subsubsection{Rewiring for Training Network Diversification}
To assess the effectiveness of our edge-rewiring strategy for diversifying the training networks, we compare MIND with the same model trained on the same networks, only without rewiring. The results are shown in Fig.~\ref{fig:fig_ablala}, with bars depicting the effect of diversifying the training set on the performance of MIND (shorter bars correspond to higher improvement) on real networks listed in Fig.~\ref{fig:main_res}. For each network, the AUC of dismantling with the diversified (rewired) training set is shown as the percentage of the AUC associated with no rewiring  (values below 100 indicate that training on rewired networks has led to a better dismantling policy).

The results demonstrate that rewiring the training networks yields an overall performance improvement on real networks. To analyze the results, we refer to the assortativity and modularity of the real networks in Table~4 in Appendix~E.
The most significant performance gains are observed for highly modular networks. For instance, in roads-california (third-to-last bar in the lower panel of Fig.~\ref{fig:fig_ablala}), which has a modularity of $0.975$, training with rewired networks led to an over $80\%$ reduction in the AUC of dismantling. This improvement is likely linked to the rewiring-induced modularity in the otherwise non-modular synthetic networks.
We also observe notable performance gains for networks with strong disassortativity. For example, in munmun\_twitter\_social (seventh-to-last bar in the upper panel of Fig.\ref{fig:fig_ablala}), which has a degree assortativity of $-0.878$, the AUC is reduced by $40\%$. Although the LPA and Copying Model naturally produce slightly disassortative networks, our diversifying rewirings enable the model to interact with a much wider range of (dis)assortative mixings and thereby significantly enhance the learned embedding and policy by increased exposure to different topologies.

\begin{figure}[t]
    \centering
    \includegraphics[width=\linewidth]{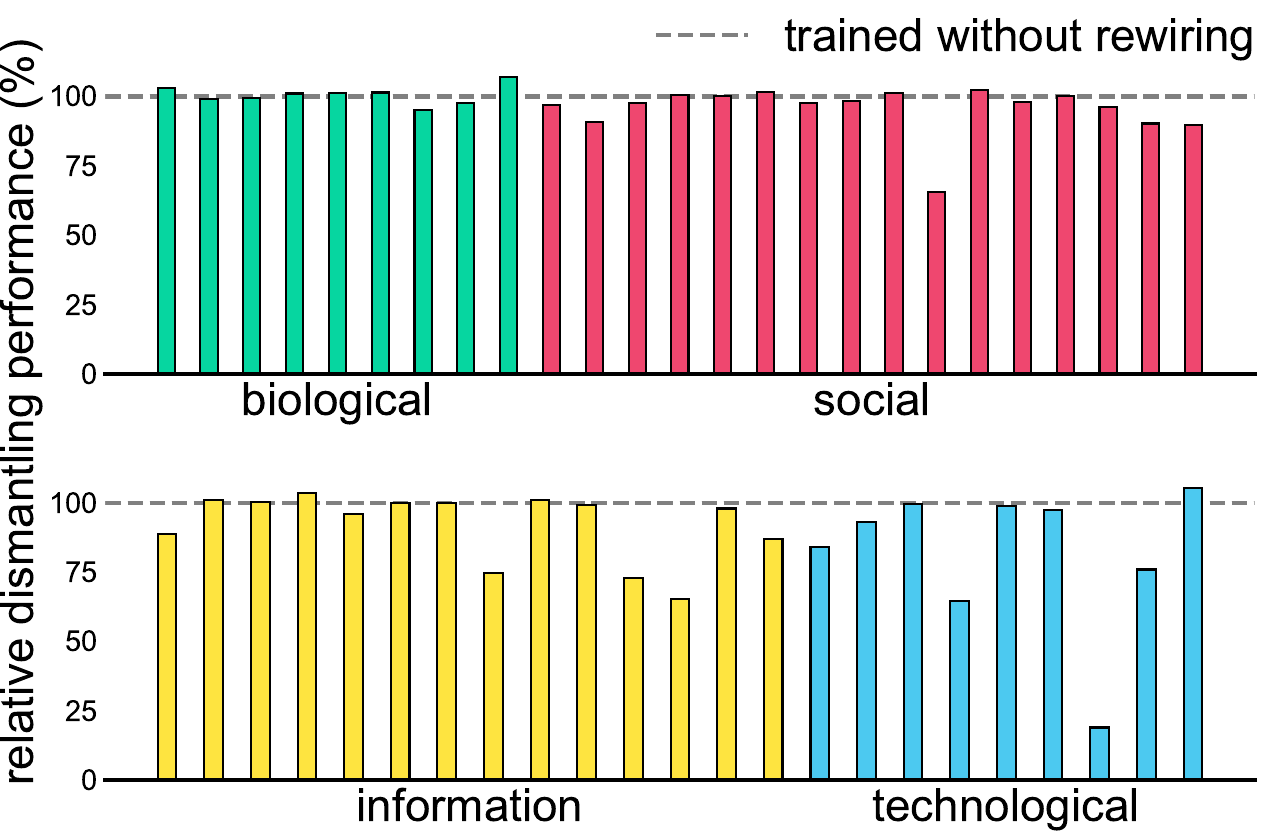}
    \caption{Performance improvement by algorithmically diversified training networks through degree-preserving edge-rewirings. Bars show the AUC of the dismantling of MIND trained on rewired networks, normalized to the baseline without rewiring (gray dotted lines).}
    \label{fig:fig_ablala}
\end{figure}

\section{Conclusion}
Eliminating the need for initializing GNNs with handcrafted features is highly sought after. Besides dropping the feature computation overhead, featureless initialization eliminates the risk of embedding bias and grants autonomy for learning more complex embeddings, which is key to finding better solutions to the downstream problems. We tackled this with two ideas: i) building an expressive message-passing framework, and ii) exposing the model to interactions with systematically diversified network geometries, facilitating the learning of complex structural roles.
The proposed model, applied to the network dismantling problem, achieved state-of-the-art performance on a comprehensive testbed of real-world networks. Note that MIND is computationally more efficient than the well-performing methods of its category, machine learning methods, as well as the well-established dismantling methods in the literature (except for EI only on dense networks). An intriguing conclusion is that the contributions of this manuscript are applicable to many important network/graph problems where an unbiased GNN embedding can be learned on synthesized diverse data and lead to breakthrough solutions.

\section{Acknowledgments} RS and HH acknowledge the Australian Research Council Discovery Project No. DP240102585, and the support from the IOTA Foundation. RS and HT acknowledge the funding by UK Research and Innovation (UKRI) under the UK government’s Horizon Europe funding guarantee (grant number 101084642). HT thanks Chi-Bach Pham for his help with some of the experiments.

\bibliography{aaai2026}


\appendix
\setcounter{secnumdepth}{1}
\setcounter{theorem}{0}

\renewcommand{\theequation}{A\arabic{equation}}
\setcounter{equation}{0}

\vspace{3em}
\hrule
\section*{Appendix}
\vspace{1em}
\hrule
\vspace{3em}

\section{Convergence of Node Embeddings}
\label{supp:ec}

Let us, without loss of generality, assume a GNN with a single attention head. For a network $G$ with $N$ nodes, the embedding of node $v_i$ after $k$ message-passing steps is denoted by $e_i^{(k)} \in \mathbb{R}^F$. The matrix containing the $k$-th step embeddings of all nodes is:
\begin{align}
    e^{(k)} = \begin{bmatrix} {e_1^{(k)}}^\top \\ \vdots \\ {e_N^{(k)}}^\top \end{bmatrix} \in \mathbb{R}^{N \times F}.
\end{align}
We also define the vectorized node embeddings as:
\begin{align}
    \mathrm{vec}(e^{(k)}) = [e_1^{(k)} \| \dots \| e_N^{(k)}] \in \mathbb{R}^{NF}.
\end{align}

Also, let $W \in \mathbb{R}^{F \times F}$ be a learnable weight matrix, and $A_\alpha \in \mathbb{R}^{N \times N}$ be an attention-enhanced adjacency matrix with elements defined as:
\begin{align}
[A_\alpha]_{i,j} =
\begin{cases}
\alpha_i & \text{if } j = i \\
\alpha_{i,j} & \text{if } j \in \mathcal{N}(i), \\
0 & \text{otherwise}
\end{cases}
\end{align}
where $\mathcal{N}(i)$ denotes the set of neighbors of node $i$. Then, the message-passing update in Eq.~(6) can be written as:
\begin{align}
e^{(k+1)} = A_\alpha e^{(k)} W.
\end{align}
\begin{theorem}
Assume $A_\alpha$ and $W$ are both diagonalizable with respective primary eigenvectors $u_1 \in \mathbb{R}^N$ and $v_1 \in \mathbb{R}^F$, whose corresponding eigenvalues are $\lambda_1$ and $\mu_1$, such that $|\lambda_1| > |\lambda_n|$ and $|\mu_1| > |\mu_n|$ for all $n \ne 1$. Then, for any initial embedding $e^{(0)}$ such that $\mathrm{vec}(e^{(0)})$ has nonzero projection onto $u_1 \otimes v_1$:
\begin{align*}
e^{(k)} \rightarrow u_1 v_1^\top \quad \text{as } k \rightarrow \infty.
\end{align*}
In particular, node embeddings $e_i^{(k)}$ converge to a scalar multiple of $v_1^\top$:
\begin{align*}
e_i^{(k)} \rightarrow [u_1]_i \cdot v_1 \quad \text{as } k \rightarrow \infty,
\end{align*}
i.e., all node embeddings align in the same direction and lie in a one-dimensional subspace of $\mathbb{R}^F$.
\label{theo_os}
\end{theorem}

\begin{proof}
The message-passing update $e^{(k+1)} = A_\alpha e^{(k)} W$ can be equivalently written in vectorized form as:
\begin{align}
    \mathrm{vec}(e^{(k+1)}) = M \cdot \mathrm{vec}(e^{(k)}),
\end{align}
where
\begin{align*}
M := A_\alpha \otimes W \in \mathbb{R}^{NF \times NF}.
\end{align*}
Since both $A_\alpha$ and $W$ are diagonalizable, so is $M$ due to properties of the Kronecker product. In particular, the eigenvectors $\{u_n \otimes v_m\}$ with corresponding eigenvalues $\lambda_n \mu_m$ form a complete basis of $\mathbb{R}^{NF}$. The initial embedding $\mathrm{vec}(e^{(0)}) \in \mathbb{R}^{NF}$ can be expressed as:
\begin{align*}
\mathrm{vec}(e^{(0)}) = \sum_{n=1}^N \sum_{m=1}^F c_{n,m} (u_n \otimes v_m),
\end{align*}
where $c_{n,m} \in \mathbb{C}$ are projection coefficients; by assumption, $c_{1,1} \ne 0$.

After $k$ iterations of message-passing, the following holds:
\begin{align*}
\mathrm{vec}(e^{(k)}) 
&= M^k \, \mathrm{vec}(e^{(0)}) \\
&= \sum_{n=1}^N \sum_{m=1}^F c_{n,m} (\lambda_n \mu_m)^k (u_n \otimes v_m) \\
&= (\lambda_1 \mu_1)^k c_{1,1} (u_1 \otimes v_1) + \\
&\sum_{(n,m) \ne (1,1)} (\lambda_1 \mu_1)^k c_{n,m} \left( \frac{\lambda_n \mu_m}{\lambda_1 \mu_1} \right)^k (u_n \otimes v_m).
\end{align*}

Since $|\lambda_n \mu_m / \lambda_1 \mu_1| < 1$ for all $(n,m) \ne (1,1)$, all non-dominant terms vanish as $k \to \infty$, yielding:
\begin{align*}
\begin{split}
\lim_{k \to \infty} \mathrm{vec}(e^{(k)}) &= c_{1,1} (\lambda_1 \mu_1)^k (u_1 \otimes v_1) \\
\lim_{k \to \infty} e^{(k)} &= c_{1,1} (\lambda_1 \mu_1)^k \cdot u_1 v_1^\top.
\end{split}
\end{align*}
Therefore, the following holds for each node's embedding $e_i^{(k)}$:
\begin{align*}
e_i^{(k)} \rightarrow [u_1]_i c_{1,1} (\lambda_1 \mu_1)^k \cdot v_1, \quad \text{as } k \rightarrow \infty,
\end{align*}
i.e., all node embeddings align to the same direction and lie in a one-dimensional subspace of $\mathbb{R}^F$ in the limit of the message-passing iteration $k$.
\end{proof}

\section{Estimating Structural Heuristics}
\label{supp:es}

In this section, we conceptually demonstrate that the node encoding $z_i$ generated by MIND is capable of capturing diverse spectral information, even when node features are initialized as constants. For example, we demonstrate that MIND with two heads can approximate the \emph{Fiedler vector}. (The eigenvector corresponding to the second smallest eigenvalue of the graph Laplacian, often referred to as the Fiedler vector, gives an embedding of nodes that reveals the optimal cut and thus is extremely relevant to network dismantling and similar problems.)

We start by defining a specific message-passing operator and analyzing its spectral properties.

\begin{lemma}
Let $G$ be an undirected, non-bipartite graph with $N$ nodes, and let $A$ and $D$ denote its adjacency matrix and diagonal degree matrix, respectively. The message-passing operator is defined as
\begin{align}
T := \frac{1}{2}\left(I + D^{-1/2} A D^{-1/2}\right),
\label{eq_mpfv}
\end{align}
which has all eigenvalues lying in the interval $(0, 1]$. Moreover, the primary eigenvector $u_1$ of $T$ associated with the eigenvalue $\lambda_1 = 1$ is $u_1 = D^{1/2} \mathbf{1}_N$.
\label{lemma_mpop}
\end{lemma}

\begin{proof}
The normalized Laplacian of $G$ can be written as $L := I - D^{-1/2} A D^{-1/2} \in \mathbb{R}^{N \times N}$. Since $G$ is undirected (i.e., $L$ is symmetric) and non-bipartite, all eigenvalues $\lambda_n$ of $L$ satisfy $\lambda_n \in [0, 2)$.

We can express $T$ in terms of $L$ as follows:
\begin{align}
T = \frac{1}{2}(I + D^{-1/2} A D^{-1/2}) = I - \frac{1}{2}L.
\end{align}
Let $u_n$ be the eigenvector of $L$ associated with the $n$-th eigenvalue $\lambda_n$, i.e., $Lu_n = \lambda_n u_n$. Then:
\begin{align}
Tu_n = \left(I - \frac{1}{2} L\right) u_n = \left(1 - \frac{\lambda_n}{2}\right) u_n,
\end{align}
such that $1 - \frac{\lambda_n}{2}$ is an eigenvalue of $T$. Since $\lambda_n \in [0, 2)$, it follows that
\begin{align}
1 - \frac{\lambda_n}{2} \in (0, 1],
\end{align}
i.e., all eigenvalues of $T$ lie in $(0, 1]$.

In particular, the largest eigenvalue of $T$ is $1$ (corresponding to the smallest eigenvalue of $L$, $\lambda = 0$). The eigenvector associated with $\lambda = 0$ for the normalized Laplacian is proportional to $D^{1/2} \mathbf{1}_N$. Hence, the primary eigenvector of $T$ is $u_1 = D^{1/2} \mathbf{1}_N$, up to normalization.
\end{proof}

Next, we show that MIND can estimate the Fiedler vector by learning to approximate specific functions in its heads. This follows from the fact that, message-passing operator $T$ defined in \eqref{eq_mpfv} has eigenvalues within the range $(0, 1]$. Consequently, repeated message-passing using $T$ gradually attenuates the components of the node embeddings $e^{(k)}$ that align with non-primary eigenvectors. Among these, the component aligned with the Fiedler vector, which corresponds to the second-largest eigenvalue of $T$, decays more slowly than all others except the primary eigenvector. As a result, as $k \rightarrow \infty$, only the primary and Fiedler vectors approximately remain, such that we can approximate the Fiedler vector by performing repeated message-passing with $T$ and subsequently removing the primary eigenvector $u_1 = D^{1/2} \mathbf{1}_N$. In particular, MIND with two heads, each with feature dimension $F=1$, is sufficient for such estimation:

\textbf{Head~1} estimates the primary eigenvector $u_1$ of $L$ by capturing the node degrees. In the first round of message-passing, MIND can learn the self-weight $W^1_\sigma = 0$ and the neighbor-weight matrix $W^1_\nu = 1$. This yields $e_i^{(1)} = d_i$, where $d_i$ is the degree of node $v_i$. Since the degree information can be preserved across all subsequent layers in Head~1 by learning the attention weights $\alpha_i^h = 1$ and $\alpha_{i,j}^h = 0$ for all heads $h$ and neighbors $j$, we assume that the primary eigenvector $\tilde{u}_1 = D^{1/2} \mathbf{1}_N$ of $T$ is stored in Head~1.

\textbf{Head~2} approximates repeated message-passing with $T$ by learning the following neural networks:
\begin{align*}
&[T]_{i, j} = \\
&\begin{cases}
\mathrm{MLP}_\sigma(d_i, \cdots) = \frac{1}{2}+\frac{1}{d_i},&j=i\\
\mathrm{MLP}_\nu(d_i, d_j, \cdots) = \frac{1}{2}+\frac{1}{\sqrt{d_i d_j}},&j \in \mathcal{N}(i)\\
0, &\text{o.w.},
\end{cases}
\end{align*}
where $d_i$ and $d_j$ are learned and stored in Head~1, and Head~2 can leverage them thanks to the (MIND-AM) mechanism proposed in this work.

The initial embedding in Head~2 is set to $1$ for every node $v_i \in V$; that is, the initial embedding vector of Head~2 containing all node embeddings is $e^{(0)} = \mathbf{1}_N$. Let $e^{(k)} \in \mathbb{R}^N$ denote the vector of node embeddings after the $k$-th message-passing iteration in Head~2. Since $T$ is diagonalizable, we can express the initial embedding vector $e^{(0)}$ as a linear combination of the eigenvectors $u_1,\dots, u_N$ of the matrix $T$:
\begin{align}
e^{(0)} = \sum_{n=1}^N c_n u_n,
\end{align}
where $c_n$ denotes the projection coefficient onto $u_n$. After $k$ iterations of message-passing, the embedding evolves as:
\begin{align}
\begin{split}
e^{(k)} &= T^k e^{(0)} = \sum_{n=1}^{N} c_n \lambda_n^k u_n \\
&= c_1 \lambda_1^k u_1 + c_2 \lambda_2^k u_2 + \epsilon \\
&= c_1 D^{1/2} \mathbf{1}_N + c_2 \lambda_2^k u_2 + \epsilon,
\end{split}
\label{eq_mpna}
\end{align}
where $\lambda_1 = 1$, and $\epsilon$ captures the residual terms that decay to zero as $k \to \infty$. Thus, $e^{(k)}$ asymptotically aligns with $u_1$ while preserving a vanishing component along the direction of $u_2$.

Given access to $e^{(k)}$ and the degree estimates from Head~1, the function $\mathrm{MLP}_f$ in (MIND-MP) can be trained to recover $u_2$ by removing the projection of $e^{(k)}$ onto $D^{1/2} \mathbf{1}_N$. Specifically:
\begin{align}
\begin{split}
\tilde{u}_2 &\approx e^{(k)} - \frac{\tilde{u}_1^\top e^{(k)}}{|\tilde{u}_1|_2^2} \tilde{u}_1 \approx C u_2,
\end{split}
\label{eq_v2approx}
\end{align}
where $C$ is a constant scalar. The inner-product $\tilde{u}_1^\top e^{(k)}$ aggregates information from all nodes. Within the message-passing framework, this is made possible by our omni-node $v_o$, which is an extra node connected to every node in the original network.

Overall, this section demonstrates that, with constant input node features, the Fiedler vector can naturally emerge from message-passing, enabled by the MIND-MP module and the omni-node. 

\section{Generating Synthetic Networks}
\label{supp:gs}

\subsubsection{Random Network Models}
To synthesize each training network, we randomly select one of three generation models---LPA, Copying Model, or ER---with equal probability. The network size $|V|$ is sampled uniformly from the range of $100$ to $200$. The average degree is determined by the parameter $m$ drawn randomly either from the set $\{1, 6, 8, 10\}$ with probability $1/3$, or from the set $\{2, 3, 4, 5\}$ with probability $2/3$. For Copying Model and LPA networks, a degree distribution (power-law) exponent $\gamma\in [2, 4]$ is also selected uniformly at random. 

The LPA model starts with a clique of $m+1$ nodes. Then, at each step, a new node forms connections to $m$ existing nodes with the probability of choosing an existing node to form a connection being proportional to $d_i + m(\gamma - 3)$, where $d_i$ is the degree of an existing node $v_i$. This results in a scale-free network with a power-law degree distribution $P(d)\sim d^{-\gamma}$ and an average degree of $\langle d \rangle = 2m$. When generating using the Copying Model, we also start with a clique of $m+1$ nodes. Each of the $m$ edges of a newly added node is formed by either i) connecting to a randomly selected existing node, with probability $\alpha = \frac{2-\gamma}{1-\gamma}$, or ii) connecting to one neighbor of a randomly selected existing node, with probability $1 - \alpha$. This also results in a scale-free network with a power-law degree distribution $P(d) \sim d^{-\gamma}$ and an average degree of $\langle d \rangle = 2m$. (LPA and Copying Model generate fundamentally different structures given the same parameters, due to their different mechanisms of edge formation.) When generating using the ER model, for all pairs of nodes, an edge will be formed with probability $p=\frac{m}{n-1}$, leading to a network whose degree distribution is Binomial (well-approximated by Poisson) with an average degree of $\langle d\rangle=m$.

\subsubsection{Degree-Preserving Rewiring}
In the next step, we perform degree-preserving rewirings to randomize the structural signature of the generator models, and induce different levels of modularity and degree-assortativity in the network. Modularity is the quality of having well-defined communities, and degree-assortativity is the tendency of nodes to be connected to other nodes with similar degrees. Networks with the same exact degree sequence can have different modularity levels (from being highly modular to the opposite extreme of being anti-modular or almost bipartite) and different degree mixing (from disassortative to uncorrelated to assortative). To control these qualities via random rewiring, as will be elaborated later, we treat modularity (resp. degree-assortativity) as the tendency of nodes of similar randomly assigned numerical labels (resp. similar degree-based assigned labels) to be connected to one another. 

The degree-preserving rewiring is performed by selecting two edges $(v_i, v_l)$ and $(v_j, v_k)$ (with four distinct end-nodes) and swapping them with $(v_i, v_k)$ and $(v_j, v_l)$ if they do not already exist. It can be seen that one double-edge swap of this form will change the neighborhood of four nodes without changing the degree of any node in the network. We also reject an edge swap if performing it disconnects the network into isolated components. Let $\mathrm{l}_i$ be the unique integer label identifying node $v_i$. Given a random double-edge rewiring $\{(v_i, v_l), (v_j, v_k)\} \rightarrow \{(v_i, v_k), (v_j, v_l)\}$, the rewiring is accepted if $(\mathrm{l}_i-\mathrm{l}_j)(\mathrm{l}_k-\mathrm{l}_l)>0$ (resp. $(\mathrm{l}_i-\mathrm{l}_j)(\mathrm{l}_k-\mathrm{l}_l)<0$) to increase (resp. decrease) the connectivity between similarly labeled nodes; this rewiring direction can be deemed as label-assortative (resp. label-disassortative) regardless of whether it is used to control modularity or degree-assortativity. The following describes the rewiring process to diversify the random training networks.

\begin{table}[t]
\centering
\vspace{1em}
\begin{tabular}{l|ccc}
\hline
statistics     & mean (std)     & min   & max    \\ \hline
size           & 149.62 (29.06) & 100   & 200    \\
avg. degree    & 7.49 (4.37)    & 1.98  & 19.45  \\
min degree     & 3.36 (2.56)    & 1.00  & 10.00  \\
max degree     & 29.17 (16.83)  & 7.00  & 102.00 \\
assortativilty & -0.06 (0.19)   & -0.51 & 0.51   \\
modularity     & 0.38 (0.15)    & 0.14  & 0.85   \\ \hline
\end{tabular}
\label{tab:train_stat}
\caption{Statistics on synthetic networks used for training.}
\vspace{1em}
\end{table}

\subsubsection{Rewiring for Structural Diversification}
With a coin flip, we choose to label the nodes either i) randomly, or ii) in the order of their degree (nodes with the same degree are given consecutive integers as labels). The former choice will lead to the process of controlling the modularity of the network, and the latter choice will lead to the control of the network's degree-assortativity. This controlling of the modularity/degree-assortativity is achieved by the iterative application of degree-preserving edge rewirings. The target (label-)assortativity coefficient is then sampled uniformly from the set $\{0.05, 0.15, 0.2, 0.25, 0.3, 0.4, 0.5\}$ and a random (negative or positive) sign is applied to the coefficient. Starting with any network, we perform a randomly picked label-assortative (resp. label-disassortative rewiring) if the network's current label-assortativity is below (resp. above) the target coefficient. Double-edges rewirings, in the appropriate direction, are iteratively performed until the network’s assortativity (computed corresponding to the chosen node labeling) matches the target value within tolerance or the maximum number of rewiring attempts is reached. Note that the resulting rewired networks will always be connected.

\begin{algorithm}[t]
\caption{MIND Training}
\label{alg:algorithm}
\textbf{Initialization}: Diversified set of training networks $\mathcal{G}$; state-action value networks $Q_i$ and corresponding target networks $Q_{\text{targ}, i}$ ($i = 1, 2$); policy network $\pi$; empty replay buffer $\mathcal{D}$; random sampling steps $s_\text{start}$; target network update frequency $f_\text{targ}$; global step $s=0$;
\begin{algorithmic}[1] 
\STATE sample $G_0\sim P_{\mathcal{G}}$; initialize episode $t=0$;
\WHILE{$s<s_\text{total}$}
\IF{$s < s_\text{start}$}
\STATE randomly choose $v_t\in V_t$;
\ELSE
\STATE $v_t=\pi(G_t)$; 
\ENDIF
\STATE $G_{t+1} = G_t\setminus \{v_t\}$;
\STATE Store transition $(G_t, v_t, r_t, G_{t+1})$ in $\mathcal{D}$;
\STATE $t\mathrel{+}=1$; $s\mathrel{+}=1$;
\IF{$\mathrm{LCC}(G_t)<0.1|V_0|$}
    \STATE Sample $G_0\sim P_{\mathcal{G}}$; reset episode $t=0$;
\ENDIF
\IF {$s > s_\text{start}$}
\STATE Randomly sample $\mathcal{B}$ from $\mathcal{D}$;
\STATE Update $Q_i$, $i=1, 2$, by gradient descent with\textsuperscript{*}:
\begin{align*}
    \frac{1}{|\mathcal{B}|}\sum_\mathcal{B}(Q_i(G_t, v_t)-(r_t+\gamma\hat{Q}'))^2;
\end{align*}
\STATE Update $\pi$ by gradient ascent with
\begin{align*}
    \frac{1}{|\mathcal{B}|}\sum_\mathcal{B}\min_{i=1,2}Q_{i}(G_t, \pi(G_t))-\alpha\log(\pi(G_t)|G_t);
\end{align*}
\IF{$\mathrm{mod}(s, f_\text{targ})=0$}
\STATE Update target networks:
\begin{align*}
    Q_{\text{targ}, i}\leftarrow Q_i, \quad i=1,2;
\end{align*}
\ENDIF
\ENDIF
\ENDWHILE
\end{algorithmic}
{\footnotesize\hspace{0.5em}* Define ${\scriptstyle\hat{Q}' = \mathbb{E}_\pi\left[\min_{i=1,2}Q_{\text{targ},i}(G_{t+1}, v_{t+1}) - \alpha\log\pi(v_{t+1} | G_{t+1})\right]}$.}
\end{algorithm}
\vspace{-0.7em}

\section{MIND Training}
\label{app:setup}

MIND is trained on an Ubuntu 22.04 server with an Intel(R) Xeon(R) w9-3475X CPU (512 GB RAM) and an RTX 5000 Ada GPU (32 GB VRAM). On this machine, the complete training process took approximately 90 h, which is comparable to the training time reported for the RL-based baseline method FINDER. MIND requires substantially less training time than the other machine learning-based baseline, GDM, which relies on brute-force dismantling of training networks with factorial complexity in the number of nodes. The pseudo-code of the training process of MIND is presented in Algorithm~\ref{alg:algorithm}.

The same set of hyperparameters (listed in Table~\ref{tab:hyperparameters}), mostly aligned with the hyperparameters for discrete SAC in (Huang et al. 2022), is used across all experiments. The Adam optimizer is used for updating the neural networks. To estimate the state-action value, we apply a forgetting factor of $\lambda = 0.99$, resulting in the following:
\begin{align}
Q(G_t, v_i) = r_t + \mathbb{E}\!\left[\sum_{k=t+1}^{|V_0|-1} \gamma^{k-t} r_k\right],
\end{align}
which slightly discounts future AUC contributions but effectively prevents unbounded growth of the cumulative AUC for large networks. To stabilize the bootstrapping estimation of $Q(G_t, v_i) $, we update the target Q-networks intermittently (every 8000 steps). By setting the target Q-network update factor to 1, in the update event, the target networks are synchronized with the main Q-networks.
The number of message-passing iterations is set to $K=6$. Although increasing the number of message-passing iterations $K$ enlarges the receptive field of the GNN encoder in MIND, it leads to higher computational complexity and increased training difficulty. Also, most real-world networks exhibit small-world properties with small diameters, motivating our choice of a small $K$. MIND uses a relatively lightweight model with 120 k trainable parameters. This is approximately 100 times smaller than the best-performing baseline, GDM, with 13 M parameters. The final embedding dimension of MIND is 96 (after concatenation of MIND-MP), compared to 1500 in GDM.

\begin{table}[t]
    \centering
    \begin{tabular}{ll}
        \hline
        \textbf{Parameter} & \textbf{Value} \\
        \hline
        Total number of steps $s_\text{total}$ & $8\times 10^6$\\
        Replay buffer size $|\mathcal{D}|$ & $2\times 10^6$\\
        Learning rate for Q-network & $3 \times 10^{-4}$ \\
        Learning rate for policy network & $3 \times 10^{-4}$ \\
        Batch size for updating $|\mathcal{B}|$ & $512$ \\
        Start learning $s_\text{start}$ & $100000$ \\
        Target Q-network update factor & $1$ \\
        Forgetting factor $\gamma$ & $0.99$ \\
        Policy network update frequency & $4$ \\
        Target network update frequency & $8000$ \\
        Node embedding size $F$ & 4\\
        Number of heads $H$ & $4$\\
        Message-passing iterations $K$& $6$\\
        $\mathrm{MLP}_\sigma$ and $\mathrm{MLP}_\nu$ structure & $16, [32], 1$\\
        $\mathrm{MLP}_\theta$ and $\mathrm{MLP}_\phi$ $\text{structure}^*$ & $196, [256, 256], 1$\\
        \hline
        \multicolumn{2}{l}{$^*$ $\mathrm{MLP}_\zeta$ is merged into $\mathrm{MLP}_\theta$ and $\mathrm{MLP}_\phi$.}\\
        \hline
    \end{tabular}
    \caption{The hyperparameters of MIND. For Neural Networks, the first number is the input size, the numbers in $[\cdot]$ are the size of the hidden layers, and the last number is the output size.}
    \label{tab:hyperparameters}
\end{table}

\section{Statistics on the Real Networks}
\label{app:g_stat}

In Table~\ref{tab:graph_stats}, we present key statistics of the real-world networks used in the Experiments section. Specifically, we report the total number of nodes $|V|$, average degree, degree-assortativity coefficient (within the range $[-1, 1]$), and modularity (within the range $[-1, 1]$) of each network. The table shows that these networks span a wide range of sizes (from $128$ to $1.4$ million nodes) and average degrees (from $2$ to $33$). Many real networks are disassortative (or exhibit negative degree-assortativity), whereas certain networks---mostly technological networks---have very high modularity.

Since the message-passing operations in MIND are confined to local neighborhoods, its computational complexity scales linearly with the number of nodes, while allowing parallelization on modern GPUs. Such efficiency enables MIND to scale to huge real-world networks. For example, MIND requires only 6.8 hours to dismantle the largest real network, Hyves, which contains 1.5 million nodes, compared to 9.8 hours for GDM and 12.2 hours for FINDER.

\section{Detailed Dismantling Results on Real-World Networks}
\label{app:per_g_real}

In Table~\ref{tab:graph_performance}, we report numerical values of the relative AUC of the dismantling curve, given by every method on all of the real-world networks. The results shown in Fig~2 are generated from the data reported in Table~\ref{tab:graph_performance}. MIND's better performance over the state-of-the-art baselines is statistically significant, with $p$-values of $<0.008$ for the Wilcoxon test comparing MIND versus each baseline.

\section{Limitation and Future Works}
Despite the state-of-the-art performance, dismantling networks with huge diameters is the weakness of MIND. This is a common weakness of all common message-passing GNNs arising from their limited message-passing reach, and remains to be solved while respecting the linear computational complexity.
Another potential future work is the RL training framework. 
In this manuscript, we did not delve into the effects of the RL model on the dismantling performance. We based our choice of the RL algorithm on observing the performance of a number of recent methods in elementary tests. Due to its off-policy learning, SAC uses the observed network dismantling steps more efficiently in training. 
Exploiting the most recent advances in RL and/or tailoring RL models (e.g., leveraging domain knowledge to reduce the policy search space) for the purpose of network dismantling is a potential direction for future work. Last but not least, the application of the ideas presented in this manuscript can be investigated to derive GNN-based solutions to general complex network problems.

\begin{table*}[t]
    \centering
    \begin{tabular}{l c c c c}
        \hline
        \textbf{Network} & $|V|$ & \textbf{avg. deg.} & \textbf{assortativity} & \textbf{modularity} \\
        \hline
        \multicolumn{5}{c}{\textbf{biological}} \\
        \hline
        arenas-meta & 453 & 9.01 & -0.214 & 0.445 \\
        dimacs10-celegansneural & 297 & 14.46 & -0.163 & 0.384 \\
        foodweb-baydry & 128 & 32.91 & -0.104 & 0.178 \\
        foodweb-baywet & 128 & 32.42 & -0.112 & 0.180 \\
        maayan-Stelzl & 1706 & 3.74 & -0.187 & 0.618 \\
        maayan-figeys & 2239 & 5.75 & -0.331 & 0.465 \\
        maayan-foodweb & 183 & 26.80 & -0.254 & 0.364 \\
        maayan-vidal & 3133 & 4.29 & -0.097 & 0.678 \\
        moreno\_propro & 1870 & 2.44 & -0.152 & 0.847 \\
        \hline
        \multicolumn{5}{c}{\textbf{social}} \\
        \hline
        advogato & 6539 & 13.24 & -0.061 & 0.461 \\
        douban & 154908 & 4.22 & -0.180 & 0.598 \\
        ego-twitter & 23370 & 2.81 & -0.478 & 0.895 \\
        hyves & 1402673 & 3.96 & -0.023 & 0.771 \\
        librec-ciaodvd-trust & 4658 & 14.22 & 0.104 & 0.434 \\
        librec-filmtrust-trust & 874 & 3.00 & 0.078 & 0.754 \\
        loc-brightkite & 58228 & 7.35 & 0.011 & 0.679 \\
        loc-gowalla & 196591 & 9.67 & -0.029 & 0.713 \\
        munmun\_digg\_reply\_LCC & 29652 & 5.72 & 0.003 & 0.406 \\
        munmun\_twitter\_social & 465017 & 3.58 & -0.878 & 0.649 \\
        opsahl-ucsocial & 1899 & 14.57 & -0.188 & 0.262 \\
        pajek-erdos & 6927 & 3.42 & -0.116 & 0.696 \\
        petster-hamster & 2426 & 13.71 & 0.047 & 0.549 \\
        slashdot-threads & 51083 & 4.60 & -0.034 & 0.483 \\
        slashdot-zoo & 79116 & 11.82 & -0.075 & 0.341 \\
        soc-Epinions1 & 75879 & 10.69 & -0.041 & 0.443 \\
        \hline
        \multicolumn{5}{c}{\textbf{information}} \\
        \hline
        cit-HepPh & 34546 & 24.37 & -0.006 & 0.725 \\
        citeseer & 384413 & 9.07 & -0.061 & 0.800 \\
        com-amazon & 334863 & 5.53 & -0.059 & 0.926 \\
        com-dblp & 317080 & 6.62 & 0.267 & 0.820 \\
        dblp-cite & 12591 & 7.88 & -0.046 & 0.633 \\
        dimacs10-polblogs & 1224 & 27.31 & -0.221 & 0.427 \\
        econ-wm1 & 260 & 19.65 & 0.032 & 0.268 \\
        linux & 30837 & 13.86 & -0.174 & 0.480 \\
        p2p-Gnutella06 & 8717 & 7.23 & 0.052 & 0.388 \\
        p2p-Gnutella31 & 62586 & 4.73 & -0.093 & 0.502 \\
        subelj\_jdk\_jdk & 6434 & 16.68 & -0.223 & 0.494 \\
        subelj\_jung-j\_jung-j & 6120 & 16.43 & -0.233 & 0.471 \\
        web-EPA & 4271 & 4.17 & -0.303 & 0.647 \\
        web-Stanford & 281903 & 14.14 & -0.112 & 0.927 \\
        \hline
        \multicolumn{5}{c}{\textbf{technological}} \\
        \hline
        eu-powergrid & 1467 & 2.48 & -0.064 & 0.926 \\
        gridkit-eupowergrid & 13844 & 2.50 & 0.014 & 0.966 \\
        gridkit-north\_america & 16167 & 2.50 & 0.050 & 0.968 \\
        internet-topology & 34761 & 6.20 & -0.215 & 0.610 \\
        london\_transport\_multiplex\_aggr & 369 & 2.33 & 0.137 & 0.829 \\
        opsahl-openflights & 2939 & 10.67 & 0.051 & 0.635 \\
        roads-california & 21048 & 2.06 & -0.002 & 0.975 \\
        roads-sanfrancisco & 174956 & 2.54 & 0.083 & 0.986 \\
        tech-RL-caida & 190914 & 6.37 & 0.025 & 0.856 \\
        \hline
    \end{tabular}
    \caption{Statistics on the real-world networks.}
    \label{tab:graph_stats}
\end{table*}

\begin{table*}[t]
    \centering
    \begin{tabular}{l | c c c c c c c c | c}
        \hline
        \textbf{Network} & AD & BC & PR & MS & EI & GND & FINDER & GDM & MIND \\
        \hline
        \multicolumn{10}{c}{\textbf{biological}} \\
        \hline
        arenas-met & 100.6 & 138.3 & 111.5 & 117.0 & 117.0 & 125.3 & 107.3 & 97.1 & 100.0 \\
        dimacs10-c & 117.3 & 142.5 & 132.5 & 133.8 & 120.3 & 96.7 & 115.5 & 113.8 & 100.0 \\
        foodweb-bd & 108.4 & 120.9 & 133.2 & 115.4 & 113.8 & 115.5 & 107.7 & 110.6 & 100.0 \\
        foodweb-bw & 107.4 & 121.7 & 133.3 & 115.2 & 112.0 & 118.9 & 106.4 & 107.8 & 100.0 \\
        maayan-Ste & 104.9 & 140.3 & 116.1 & 114.3 & 116.9 & 147.6 & 103.2 & 102.4 & 100.0 \\
        maayan-fig & 103.7 & 157.9 & 131.0 & 133.2 & 160.5 & 107.4 & 102.2 & 103.0 & 100.0 \\
        maayan-foo & 120.2 & 132.2 & 124.5 & 162.0 & 154.9 & 116.8 & 104.5 & 104.8 & 100.0 \\
        maayan-vid & 107.2 & 128.6 & 113.2 & 118.4 & 105.4 & 115.2 & 101.0 & 103.7 & 100.0 \\
        moreno\_pro & 124.1 & 157.6 & 135.5 & 156.0 & 94.4 & 124.4 & 114.6 & 107.4 & 100.0 \\
        \hline
        \multicolumn{10}{c}{\textbf{social}} \\
        \hline
          advogato & 107.6 & 121.1 & 158.5 & 119.0 & 117.8 & 113.9 & 103.9 & 105.6 & 100.0 \\
            douban & 108.0 & 117.4 & 102.3 & 126.9 & 115.4 & 115.6 & 95.0 & 95.7 & 100.0 \\
        ego-twitte & 110.3 & 135.0 & 109.1 & 169.4 & 104.3 & 118.3 & 112.2 & 101.2 & 100.0 \\
             hyves & 111.5 & 159.6 & 104.2 & 135.3 & 169.1 & 110.7 & 102.6 & 100.5 & 100.0 \\
        librec-cia & 132.4 & 129.3 & 135.6 & 142.4 & 145.6 & 127.3 & 116.1 & 112.5 & 100.0 \\
        librec-fil & 122.8 & 166.0 & 137.6 & 176.1 & 117.8 & 113.7 & 125.2 & 109.8 & 100.0 \\
        loc-bright & 112.4 & 139.1 & 118.8 & 121.7 & 111.4 & 114.2 & 109.7 & 113.9 & 100.0 \\
        loc-gowall & 106.5 & 148.8 & 114.8 & 112.3 & 101.5 & 110.0 & 105.5 & 106.5 & 100.0 \\
        munmun\_dig & 104.7 & 124.9 & 113.1 & 105.1 & 116.2 & 123.5 & 103.3 & 106.2 & 100.0 \\
        munmun\_twi & 101.3 & 103.3 & 90.7 & 138.4 & 124.9 & 103.6 & 102.9 & 98.5 & 100.0 \\
        opsahl-ucs & 107.7 & 118.4 & 114.2 & 117.3 & 127.8 & 131.6 & 105.0 & 107.8 & 100.0 \\
        pajek-erdo & 103.8 & 107.1 & 103.7 & 123.3 & 120.3 & 112.7 & 102.5 & 100.4 & 100.0 \\
        petster-ha & 128.4 & 129.6 & 141.4 & 174.5 & 108.7 & 96.8 & 129.2 & 104.6 & 100.0 \\
        slashdot-t & 103.6 & 118.9 & 108.8 & 122.5 & 127.5 & 104.2 & 101.8 & 104.1 & 100.0 \\
        slashdot-z & 103.8 & 134.6 & 112.1 & 121.9 & 131.2 & 107.6 & 101.9 & 108.4 & 100.0 \\
        soc-Epinio & 100.7 & 120.3 & 104.1 & 117.4 & 115.6 & 99.6 & 97.0 & 99.8 & 100.0 \\
        \hline
        \multicolumn{10}{c}{\textbf{information}} \\
        \hline
         cit-HepPh & 121.5 & 138.3 & 133.7 & 125.3 & 115.6 & 90.5 & 120.1 & 121.0 & 100.0 \\
          citeseer & 109.3 & 165.7 & 128.0 & 118.9 & 106.0 & 109.2 & 107.0 & 106.9 & 100.0 \\
        com-amazon & 123.1 & 172.1 & 122.6 & 152.8 & 75.5 & nan & 122.2 & 107.0 & 100.0 \\
          com-dblp & 114.3 & 121.0 & 109.5 & 180.0 & 88.5 & 106.9 & 115.6 & 97.5 & 100.0 \\
         dblp-cite & 124.8 & 130.8 & 126.1 & 153.7 & 122.8 & 124.5 & 117.0 & 110.0 & 100.0 \\
        dimacs10-p & 111.5 & 125.4 & 122.9 & 117.5 & 126.2 & 117.5 & 107.8 & 109.3 & 100.0 \\
          econ-wm1 & 103.4 & 125.0 & 123.4 & 101.5 & 120.6 & 122.8 & 92.4 & 94.2 & 100.0 \\
             linux & 128.6 & 404.0 & 194.7 & 166.0 & 93.5 & 108.3 & 124.2 & 110.7 & 100.0 \\
        p2p-Gnutel & 115.5 & 126.0 & 118.8 & 115.8 & 115.5 & 136.8 & 112.2 & 106.4 & 100.0 \\
        p2p-Gnutel & 110.7 & 131.5 & 112.0 & 110.8 & 114.3 & 135.6 & 107.4 & 101.5 & 100.0 \\
        subelj\_jdk & 112.6 & 339.3 & 141.0 & 141.4 & 110.6 & 105.2 & 102.2 & 97.8 & 100.0 \\
        subelj\_jun & 124.1 & 340.3 & 153.2 & 145.5 & 120.4 & 103.8 & 109.5 & 101.7 & 100.0 \\
           web-EPA & 109.3 & 147.3 & 112.2 & 169.0 & 150.9 & 158.6 & 106.7 & 106.9 & 100.0 \\
        web-Stanfo & 209.8 & 239.9 & 173.3 & nan & 65.1 & 114.8 & 207.2 & 98.9 & 100.0 \\
        \hline
        \multicolumn{10}{c}{\textbf{technological}} \\
        \hline
        eu-powergr & 151.5 & 190.5 & 196.5 & 317.5 & 80.6 & 82.8 & 161.7 & 109.1 & 100.0 \\
        gridkit-eu & 136.0 & 175.8 & 132.2 & 256.1 & 36.4 & 64.0 & 132.2 & 97.6 & 100.0 \\
        gridkit-no & 124.7 & 149.4 & 133.5 & 184.1 & 26.3 & 56.4 & 137.0 & 95.0 & 100.0 \\
        internet-t & 102.0 & 126.5 & 112.4 & 142.6 & 117.2 & 98.4 & 102.2 & 102.9 & 100.0 \\
        london\_tra & 129.2 & 152.9 & 134.1 & 124.3 & 112.0 & 93.5 & 129.0 & 105.4 & 100.0 \\
        opsahl-ope & 135.2 & 131.3 & 130.9 & 167.1 & 116.8 & 107.6 & 120.5 & 106.2 & 100.0 \\
        roads-cali & 192.8 & 882.0 & 125.1 & 114.9 & 23.4 & 78.0 & 116.3 & 92.2 & 100.0 \\
        roads-sanf & 184.5 & 51.6 & 258.8 & 160.4 & 15.5 & 30.0 & 176.4 & 92.0 & 100.0 \\
        tech-RL-ca & 120.7 & 172.6 & 121.3 & 157.3 & 78.5 & 112.1 & 116.2 & 106.9 & 100.0 \\
        \hline
        \textbf{Overall} & \textbf{119.9} & \textbf{167.8} & \textbf{129.5} & \textbf{142.8} & \textbf{108.0} & \textbf{109.1} & \textbf{115.0} & \textbf{104.2} & \textbf{100.0} \\
        \hline
    \end{tabular}
    \caption{Detailed dismantling results on real-world networks.}
    \label{tab:graph_performance}
\end{table*}

\end{document}